\documentclass[article]{article}
\usepackage{uai2020_camera}
\usepackage[margin=1in]{geometry}
\usepackage{times}
\usepackage{amsmath}
\usepackage{multirow} 
\usepackage{epsfig} 
\usepackage{amssymb, amsthm}
\usepackage{mathtools}
\usepackage{bm}
\usepackage{bbm}
\usepackage{graphicx}
\usepackage{hyperref}
\usepackage{xcolor}
\usepackage{enumitem}
\usepackage{microtype}
\usepackage{subfigure}
\usepackage{booktabs}
\usepackage{natbib}

\hypersetup{
    colorlinks,
    linkcolor={red!90!black},
    citecolor={blue!50!black},
    urlcolor={blue!80!black}
}

\newcommand{\w}{\mathbf{w}}

\newcommand{\s}{\mathbf{s}}
\newcommand{\x}{\mathbf{x}}

\newcommand{\f}{\mathbf{f}}
\newcommand{\g}{\mathbf{g}}

\newcommand{\R}{\mathbb{R}}
\newcommand{\T}{\mathbf{T}}
\newcommand{\A}{\mathbf{A}}

\newcommand{\J}{\mathbf{J}}

\newcommand{\W}{\mathbf{W}}

\newtheorem{Definition}{Definition}

\newtheorem{Theorem}{Theorem}

\newtheorem{Lemma}{Lemma}
\newcommand{\Lambdab}{\mathbf{L}}
\newcommand{\lambdab}{\boldsymbol{\lambda}}
\newcommand{\thetab}{\boldsymbol{\theta}}

\newcommand{\mub}{\boldsymbol{\mu}}

\newcommand{\h}{\mathbf{h}}

\newcommand{\ct}{c^{\left(t\right)}}

\newcommand{\xt}{\x^{\left(t\right)}}
\newcommand{\st}{\s^{\left(t\right)}}

\newcommand{\diag}{\mathop{\mathrm{diag}}}

\DeclareMathOperator*{\argmax}{arg\,max}

\title{Hidden Markov Nonlinear ICA: \\ Unsupervised Learning from  Nonstationary Time Series}

\author{{\bf Hermanni H\"{a}lv\"{a}} \\
University of Helsinki \\
\And
{\bf Aapo Hyv\"{a}rinen}  \\
Universit\'{e} Paris-Saclay, Inria\\
University of Helsinki}

\begin{document}

\maketitle

\begin{abstract}
Recent advances in nonlinear Independent Component Analysis (ICA) provide a principled framework for unsupervised feature learning and disentanglement. The central idea in such works is that the latent components are assumed to be independent conditional on some observed auxiliary variables, such as the time-segment index. This requires manual segmentation of data into non-stationary segments which is computationally expensive, inaccurate and often impossible. These models are thus not fully unsupervised. We remedy these limitations by combining nonlinear ICA with a Hidden Markov Model, resulting in a model where a latent state acts in place of the observed segment-index. We prove identifiability of the proposed model for a general mixing nonlinearity, such as a neural network. We also show how maximum likelihood estimation of the model can be done using the expectation-maximization algorithm. Thus, we achieve a new nonlinear ICA framework which is unsupervised, more efficient, as well as able to model underlying temporal dynamics.
\end{abstract}

\section{INTRODUCTION}
Representation learning -- the task of finding useful features from data -- is one of the main challenges in unsupervised learning. Recent theoretical and practical advances in Nonlinear ICA provide a principled approach to this problem \citep{hyvarinen_unsupervised_2016, hyvarinen_nonlinear_2017, hyvarinen_nonlinear_2019, khemakhem_variational_2020, sorrenson_disentanglement_2020}. These works frame Nonlinear ICA as deep generative models, which allows them to harness deep neural networks to recover latent independent components from observed data. Identifiability of the latent components can be guaranteed by explicitly defining probabilistic generative models with appropriate conditional independence structures. A general framework was proposed recently by \citet{hyvarinen_nonlinear_2019}, who assumed that the components are independent given some other observed auxiliary variable. For example, in time-series data this can be the time-index or segment-index if the data is non-stationary, as was earlier assumed in Time-Contrastive Learning or TCL \citep{hyvarinen_unsupervised_2016}. Non-stationarity is a fundamental property of many applications, since for example, video, audio, and most neuroscience data are non-stationary.

A crucial limitation of all of the above nonlinear ICA models is that the conditioning auxiliary variable is always assumed observable. In some sense, these models are therefore not fully unsupervised. If, for instance, we wish to exploit the nonstationary temporal structure optimally in estimating independent components, TCL would require segment indices that correspond to the different latent data generative states. In practice we don't observe such states so the default approach is to manually segment the data.

In general, it is unrealistic to assume that we can infer from observed data alone the exact time-points at which the latent data distribution changes. In fact, such change-points may not exist at all. Segmenting data manually is also infeasible for large datasets. The default approach is therefore to segment data at equal intervals, however, this is problematic for various reasons. Consider, for example, a situation where the true latent state switches between five different states. By segmenting the data at equal intervals we will end up with an unnecessarily large number of states where just a few would have done the job. This is computationally expensive, inaccurate and will completely miss out on temporal dynamics in situations where the latent states repeat over time.

In fact, often a reasonable assumption is that non-stationarity can be succinctly summarized using a limited number of segment indices or latent states, and properly modelling such state switching is likely to improve learning. Notice that even if ground-truth nonstationary information was available, the existing methods lack the machinery to perform inference on latent temporal dynamics. In many applications, for example brain imaging, describing the dynamics in terms of latent states could be very useful in its own right.

The points above highlight the need for a nonlinear ICA method that is able to cluster observations and learn latent states and their temporal dynamics in an unsupervised fashion. A well-known approach to modelling hidden latent states in time series is to use a Hidden Markov Model (HMM). HMMs can be viewed as probabilistic mixture models where the discrete latent states, which specify the data generating distribution, are time-dependent with Markov dynamics. HMMs are especially well suited for modelling non-stationary data as they automatically allow for a representation of the time series in terms of a discrete number of states.

In this work, we therefore resolve the above limitations by combining Nonlinear ICA with a HMM. This idea has been proposed earlier for linear ICA \citep{penny_hidden_2000, zhou_mixture_2008} but their identifiability and estimation results do not directly extend to the nonlinear case. In our model, we achieve this by having the latent state act in place of the conditioning auxiliary variable in the framework of \citet{hyvarinen_nonlinear_2019}. Importantly, we are able to prove that Hidden Markov Nonlinear ICAs are identifiable. Attaining identifiability has been a major research focus for both Nonlinear ICA \citep{hyvarinen_nonlinear_2017, hyvarinen_nonlinear_2019} and HMMs \citep{allman_identifiability_2009, gassiat_inference_2016}, and therefore much of our paper is devoted to combining these two research strands. To the best of our knowledge, this is the first fully unsupervised non-linear ICA, in the sense that the model's identifiability comes from an \textit{unobserved} conditioning variable which is inferred from the time series as a part of learning. We further show how the structure of the model allows us to use the Expectation-Maximization (EM) algorithm for parameter estimation. In practice the Hidden Markov Nonlinear ICA is endowed with rich representation learning capabilities that allow it to simultaneously extract independent components and to learn the dynamics of the latent state that drives non-stationarity data, as illustrated by our simulations.

\section{BACKGROUND}
We start by giving an overview of the problem of unidentifiability in both nonlinear ICA and HMMs, and recently proposed solutions. For both types of models, identifiability arises as a consequence of appropriate temporal structures which suggests a natural synthesis between the two.

\subsection{NONLINEAR ICA AND IDENTIFIABILITY}

Consider a parametric model of observed data $\x$ with marginal likelihood $p_{\thetab}(\x)$. This model is \textit{identifiable} if it fulfils below:
\begin{align}
	p_{\thetab}(\x) = p_{\thetab'}(\x) \Rightarrow \thetab = \thetab' : \, \, \forall (\thetab, \thetab')
\end{align}
In the context of a latent variable model, this is closely connected to the idea of being able to recover the original latent variables, as discussed by \citet{khemakhem_variational_2020}.

Assume we observe $N$-dimensional data at discrete time-steps, $\xt = (x_1^{(t)}, \dots,  x_N^{(t)})$. Simple nonlinear ICA can be defined as the task of estimating an unobserved $N$-dimensional independent component vector $\st = (s_1^{(t)}, \dots,  s_N^{(t)})$ such that $p(\st)=\prod_{i=1}^N p(s_i)$, as well as the inverse of a mixing function $\f$, that has generated the observed data:
\begin{align}
	\xt = \f(\st)
	\label{eq:mix}
\end{align}
Unfortunately, without a temporal structure, that is if $\xt$ are i.i.d over the time-index, and if there are no constraints on $\f$, then this model is unidentifiable \citep{hyvarinen_nonlinear_1999}. In fact, the authors show that there are infinite potential nonlinear transformations and independent components that would satisfy the model, with no criterion for choosing one of them over the others. 

In order to make the model identifiable, constraints are thus needed. For time-series data, this comes naturally by placing restrictions on the temporal structure of the model. For linear ICA this approach has been shown to yield identifiable models \citep{belouchrani_blind_1997, tong_indeterminacy_1991}, and extensions to the nonlinear case have been also been proposed in earlier work \citep{harmeling_kernel-based_2003, sprekeler_extension_2014}. The first fully rigorous proof of an identifiable nonlinear ICA model, along with an estimation algorithm (Time-Contrastive Learning or TCL), was given by \cite{hyvarinen_unsupervised_2016}. The constraint imposed in that work is that of a non-stationary data generative process such that independent component vectors within different time-segments have different distributional parameters. Specifically, the model assumes that each independent component has an exponential family distribution, where the time segment index $\tau$ modulates the natural parameters (denoted as $\lambdab$):
\begin{align}
	p_{\tau}(s_i) = \frac{q_i(s_i)}{Z(\lambdab_i)}\exp\{\langle \lambdab_i(\tau), \T(s_i) \rangle\}
\end{align}
where $q_i$ is the base measure and $\T$ are the sufficient statistics. TCL then assumes that the independent components in all the segments are transformed into observed variables by some mixing function \eqref{eq:mix}. The authors prove identifiability up to a linear transformation of pointwise functions of the components:
\begin{align}
	\T(\st) = \A \h(\xt; \thetab) +\mathbf{b}
\end{align}
By learning to contrast between the different segments, the TCL algorithm learns the inverse of the mixing function and the independent components.

This seminal work has inspired other frameworks for identifiable nonlinear ICA estimation. Permutation Contrastive Learning \citep{hyvarinen_nonlinear_2017}, for instance, exploits temporal dependencies, rather than non-stationarity, to identify independent components. The unifying tenet of these identifiable nonlinear ICA algorithms is that independent components are \textit{conditionally independent} given some observed auxiliary variable. This general idea was formalized in \cite{hyvarinen_nonlinear_2019}, of which both the TCL (segment index as auxiliary variable) and the PCL (past data as auxiliary variable) are special cases.  

These identifiable nonlinear ICA models provide a principled approach to finding meaningful data representations. This is in contrast to the majority of recent deep generative models used for representation learning, such as VAEs \citep{kingma_auto-encoding_2014} and GANs \citep{goodfellow_generative_2014}, which are all malaised by unidentifiability. In fact, any generative latent variable model with an unconditional prior is unidentifiable. This issue is portrayed in depth by \cite{khemakhem_variational_2020} who resolve it by introducing identifiable VAE (iVAE). Like regular VAE, this model estimates a full generative model $p_{\thetab}(\x, \textbf{s})$, but with a factorial conditional prior $p_{\thetab}(\textbf{s} | \textbf{u})$. As in \cite{hyvarinen_nonlinear_2019}, $\textbf{u}$ is some auxiliary variable, and iVAE provides a novel algorithm to estimate nonlinear independent components in the same identifiable framework. iVAE however suffers from the same problems as TCL, as its auxiliary variable $\textbf{u}$ has to be observed. 

\subsection{HIDDEN MARKOV MODELS AND IDENTIFIABILITY}
In order to define HMMs, let $\xt \in \R^n$ be an observed random variable from a time series with a discrete time index $t \in \{1, \dots, T\}$. In a standard hidden Markov model, distribution of the observations depends conditionally on a discrete latent state random variable $\ct$ as per $p(\xt|\ct)$; we refer to this as the emission distribution. The latent state $\ct$ undergoes first-order Markov process governed by a $C \times C$ transition-probability matrix $\A$. $A_{i, j}$ is used to denote the probability of transitioning from state $c^{(t)}=i$ to $c^{(t+1)}=j$, and $\pi(c^{(1)})$ the starting-state probabilities.  The likelihood of a typical HMM is hence given by:
\begin{align}
	&p(\x^{(1)},\dots, \x^{(T)}; \A) \nonumber =\\ &\sum_{c^{(1)}, \dots,c^{(T)}} \pi(c^{(1)})p(\x^{(1)}|c^{(1)}) \prod_{t=2}^T A_{c^{(t-1)},c^{(t)}}p(\x^{(t)}|c^{(t)})
	\label{eq:hmm_likeli}
\end{align}
HMMs can be viewed as mixture models where the latent state is coupled across time by a Markov process. This observation raises the question of identifiability since mixture models with non-parametric emission distributions are generally unidentifiable, though many commonly used parametric forms are identifiable \citep{allman_identifiability_2009}. Recently, however, \cite{gassiat_inference_2016} have proven a major result that nonparametric HMMs are in general identifiable under some mild assumptions. We will use this result later and thus reproduce it here (notice that their nonparametric result subsumes parametric HMMs): 

\begin{Theorem} \citep{gassiat_inference_2016} \label{T:gassiat}
Assume the number of latent states, $C$ , is known. Use $\mu_1, \dots, \mu_C \in \R^N$ to denote nonparametric probability distributions of the $C$ emission distributions. Also assume that the transition-matrix $\A$ is full rank. Then the parameters $\A$ and $M = (\mu_1, \dots, \mu_C)$ are identifiable given the distribution, $\mathbb{P}_{\A, M}^{(3)}$, of at least 3 consecutive observations $\x^{(t-1)}, \x^{(t)}, \x^{(t+1)}$, up to label swapping of the hidden states, that is: if $\widehat{\A}$ is a $C \times  C$ transition matrix, if $\widehat{\pi}(c)$  is a stationary distribution of $\widehat{\A}$ with $\widehat{\pi}(c) > 0$ $\forall c \in \{1, \dots, C \}$, and if  $\hat{M} = (\hat{\mu}_1, \dots, \hat{\mu}_C)$ are $C$ probability distributions on $\R^N$ that verify the equality of the HMM distribution functions $\mathbb{P}_{\widehat{\A}, \hat{M}}^{(3)} = \mathbb{P}_{\A, M}^{(3)}$, then there exists a permutation $\sigma$ of the set $\{1,\dots, C\}$ such that for all $k, l = 1, \dots, C$ we have $\hat{A}_{k, l} = A_{\sigma (k), \sigma (l)}$ and $\hat{\mu}_k = \mu_{\sigma (k)}$.
\end{Theorem}

Much like for TCL, identifiability in nonparametric HMMs is a result of temporal structure, namely observations across time are independent conditionally on the latent state---this is in contrast to simple (i.i.d.) mixture models for which such general identifiability results are not available. We show below that this temporal structure of the HMM's, together with nonstationarity similar to TCL, combine to identify the resulting Hidden Markov Nonlinear ICA model.

\section{IDENTIFIABLE NONLINEAR ICA FOR NONSTATIONARY DATA}

In this section, we propose a combination of a hidden Markov model and nonlinear ICA. Specifically, we propose an HMM which has nonlinear ICA as its emission model, and show how to estimate it by Expectation-Maximization.

\subsection{MODEL DEFINITION} \label{sec:model_def}
To incorporate nonlinear ICA into the standard HMM of \eqref{eq:hmm_likeli} we define the emission distribution $p(\xt|\ct)$ as a deep latent variable model.  
First, the latent independent component variables $\st \in \R^N$ are generated from a factorial exponential family prior, given the hidden state $\ct$, as
\begin{align}
	&p(\st|\ct; \lambdab_{\ct})= \prod_{i=1}^N p(s_i^{(t)}|\ct; \lambdab_{i, \ct}) \nonumber \\ &=\prod_{i=1}^N \frac{h(s_i^{(t)})}{Z(\lambdab_{i,\ct})}\exp\{\langle \lambdab_{i,\ct},  \T_i(s_i) \rangle\}
	\label{eq:sourc_dist}
\end{align}
where $h(.)$ are the base measures, $Z(\lambdab_{i,\ct})$ the normalizing constants, and $\T_i:\R \rightarrow \R^V$ the sufficient statistics. Second, the observed data is generated by a nonlinear mixing function as in Eq.~(\ref{eq:mix}).

For remainder of the paper we assume that the exponential family is in minimal representation form so that the sufficient statistics are linearly independent. The corresponding $V$-dimensional parameter vectors are denoted by $\lambdab_{i,\ct}$. The subscripts indicate that the parameters of the $N$ different components are modulated directly, and independently, by the HMM latent state. Indeed, it is precisely this conditional dependence of the parameters on the discrete latent state that seeps through our model and generates non-stationary observed data. Note that the parameters themselves are time-homogeneous, that is they are constant over time; instead, the latent state evolves over time and determines which set of parameters is active a point in time. In other words, non-stationary arises purely from the dynamics of the latent state $\ct$. The full set of parameters for the independent components can hence be captured by a $C\times NV$ matrix $\Lambdab$ (plus the transition probabilities of the hidden states).

The nonlinear mixing function $\f$ in Eq.~(\ref{eq:mix}) is assumed to be bijective with inverse given by $\st = \g(\xt)$. It follows that in our model the conditional emission distribution for observed data is:
\begin{align}
	&p(\xt|\ct; \f, \lambdab_{\ct})= \nonumber \\ &|\J\g(\xt)|\frac{H(\g(\xt))}{Z(\lambdab_{\ct})}\exp\{\langle\lambdab_{\ct}, \T(\g(\xt))\rangle\}
	\label{eq:emission}
\end{align}
where $|\J\g(\xt)|$ is short-hand notation for the absolute value of the determinant of the Jacobian of the inverse (demixing) function, and $H(\g(\xt))=\prod_{i=1}^N h(g_i(\xt))$. We have also simplified notation by stacking the vectors for different components $\T = (\T_1, \dots, \T_N)^T$ and $\lambdab_{\ct} = (\lambdab_{1, \ct}, \dots, \lambdab_{N, \ct})^T$.

We allow $\f$ to be any arbitrary but bijective function. In practice, it can be learned as a neural network. The model can therefore be viewed as a deep generative model for non-stationary data. Finally, using $\thetab = \{\f, \Lambdab \}$ and $\Theta = \{\thetab, \A\}$ our hidden Markov nonlinear ICA model's data-likelihood is given as:
\begin{align}
	&p(\x^{(1)},\dots, \x^{(T)}; \Theta) \nonumber =\\ &\sum_{c^{(1)}, \dots,c^{(T)}} \Big( \pi(c^{(1)})p(\x^{(1)}|c^{(1)}, \thetab_{c^{(1)}}) \times \nonumber \\ 
	&\prod_{t=2}^T A_{c^{(t-1)},c^{(t)}}p(\x^{(t)}|c^{(t)}; \thetab_{\ct})\Big)
	\label{eq:model_likeli}
\end{align}
where the emission distributions in Eq.~(\ref{eq:emission}) should be plugged in.

\subsection{ESTIMATION} \label{sec:estimation}
Assume we have a sequence of observed data $\mathcal{D} = \{\x^{(1)}, \x^{(2)}, \dots, \x^{(T)}\}$ generated by \eqref{eq:model_likeli}. In order to estimate the model parameters in practice we will choose the factorial prior in \eqref{eq:sourc_dist} from a well-known family such that the normalizing constant is tractable, such as a Gaussian location-scale family. Intractable normalizing constant would make estimation very difficult, even by approximate inference methods such as Variational Bayes or VAEs. However, notice that the choice of distribution for the latent prior does not severely limit the type of data that can be modelled since the non-linear mixing function can be any arbitrary function.

Tractable exponential families also make it easy to estimate the model parameters by maximizing the likelihood in \eqref{eq:model_likeli} by the EM algorithm. The "free-energy" EM lower bound for our model is given by:
\begin{align}
	&\mathcal{L}(q(\mathbf{c}), \Theta) := \nonumber \\ &\mathbb{E}_{q(\mathbf{c})}\left[ \log p(\mathbf{c}, \x^{(1)},\dots, \x^{(T)}; \Theta) \right] - \mathbb{E}_{q(\mathbf{c})}\left[\log q(\mathbf{c}) \right]
	\label{eq:lower_bound}
\end{align}
where $\mathbf{c} = (c^{(1)},\dots, c^{(T)})$, such that the first RHS terms is the complete-data likelihood under some distribution $q(\mathbf{c})$. In the E-step one finds $q(\mathbf{c}_{\star}) := \argmax_{q(\mathbf{c})} \mathcal{L}(q(\mathbf{c}), \Theta) = p(\mathbf{c}| \x^{(1)},\dots, \x^{(T)}; \Theta)$ which is the standard result for HMMs and can be easily computed using the forward-backward (Baum-Welch) algorithm. In the M-step we aim to find $\Theta_{\star} = \argmax_{\Theta} \mathcal{L}(q(\mathbf{c_{\star}}), \Theta)$, which reduces to maximizing:
\begin{align}
	&\mathcal{\tilde{L}} (q(\mathbf{c}), \Theta):=  \mathbb{E}_{q(\mathbf{c_{\star}})}\left[ \log p(\mathbf{c}, \x^{(1)},\dots, \x^{(T)}; \Theta) \right] \nonumber =\\
	& \sum_{t=1}^T \mathbb{E}_{q(c_{\star}^{(t)})}\left[\log p(\x^{(t)}|c^{(t)};\thetab_{c^{(t)}}) \right] \nonumber \\
	&+ \sum_{t=2}^T \mathbb{E}_{q(c_{\star}^{(t-1)}, c_{\star}^{(t)})}\left[\log A_{c^{(t-1)},c^{(t)}}\right]
\end{align}
where we have left out the initial-state probability term as we can assume a stationary process and infer them directly from $\A$ as its left eigenvector. The M-step updates for $\A$ are standard:
\begin{align}
	A_{i,j}^{\star} \leftarrow \frac{\sum_{t=2}^T q(c_{\star}^{(t-1)}=i, c_{\star}^{(t)}=j)}{\sum_{t=1}^T q(c_{\star}^{(t)})}
\end{align}
M-step updates for the parameters $\Lambdab$ also follow from standard EM results for exponential families:
\begin{align}
	&\nabla_{\lambdab_k}\sum_{t=1}^T \mathbb{E}_{q(c_{\star}^{(t)})}\left[\log p(\x^{(t)}|c^{(t)};\thetab_{c^{(t)}}) \right] \nonumber \\
	&=\nabla_{\lambdab_k}\sum_{t=1}^T \mathbb{E}_{q(c_{\star}^{(t)})}\left[\langle \lambdab_k, \T(\g(\xt)) \rangle - \log Z(\lambdab_k) \right] \nonumber \\
	&=\sum_{t=1}^T q(c_{\star}^{(t)}=k)\left[\T(\g(\xt)) - \frac{\nabla_{\lambdab_k}Z(\lambdab_k)}{Z(\lambdab_k)} \right] =0 \nonumber\\
	&\Rightarrow \frac{\nabla_{\lambdab_k}Z(\lambdab_k)}{Z(\lambdab_k)} = \frac{\sum_{t=1}^T q(c_{\star}^{(t)}=k)\T(\g(\xt))}{\sum_{t=1}^T q(c_{\star}^{(t)}=k)}
	\label{eq:m_step}
\end{align}
where LHS can be rewritten as:
\begin{align}
	&\frac{1}{Z(\lambdab_k)}\nabla_{\lambdab_k} \int \left( |\J\g(\xt)|H(\g(\xt)) \right. \nonumber \\ 
	&\left. \times \exp\{\langle \lambdab_{\ct}, \T(\g(\xt)) \rangle\} \right) \nonumber\\
	&= \mathbb{E}_{p(\x^{(t)}|c^{(t)};\thetab_{c^{(t)}})}\left[\T(\g(\xt))\right]
\end{align}
Thus the M-step updates for $\lambdab_k^{\star}$ are the ones that solve:
\begin{align}
	&\mathbb{E}_{p(\x^{(t)}|k;\lambdab_k^{\star}, \f)}\left[\T(\g(\xt))\right] \nonumber \\
	&= \frac{\sum_{t=1}^T q(c_{\star}^{(t)}=k)\T(\g(\xt))}{\sum_{t=1}^T q(c_{\star}^{(t)}=k)}
	\label{eq:m_step_updates}
\end{align}
In practice, \eqref{eq:m_step_updates} has closed-form updates for many usual exponential family members. As an example, if we were to use a Gaussian distribution, the updates for mean and variance vectors would be:
\begin{align}
	&\mub_k^{\star} \leftarrow \frac{\sum_{t=1}^T q(c_{\star}^{(t)}=k) \g(\xt)}{\sum_{t=1}^T q(c_{\star}^{(t)}=k)} \nonumber \\
	&\boldsymbol{\sigma}_k^{2 \star} \leftarrow \diag \left(\frac{\sum_{t=1}^T q(c_{\star}^{(t)}=k)\textbf{y}_k^{\star} \textbf{y}_k^{\star,T}}{\sum_{t=1}^T q(c_{\star}^{(t)}=k)}\right)
\end{align}
where $\textbf{y}_k^{\star} = \g(\xt)-\mub_k^{\star}$

Next, the demixing function is estimated by parameterizing it as a deep neural network but for notational simplicity we will not write these parameters explicitly and instead subsume them in $\g$. Since an exact M-step is not possible, a gradient ascent step on the lower bound is taken instead, where the gradient is given by:
\begin{align}
	& \nabla_{\g} \mathcal{\tilde{L}} (q(\mathbf{c}), \Theta) = \nabla_{\g}\sum_{t=1}^T \mathbb{E}_{q(c_{\star}^{(t)})}\left[\log p(\xt|\ct;\thetab_{c^{(t)}}) \right] \nonumber \\
	&=\nabla_{\g}\sum_{t=1}^T \log|\J\g|\nonumber \\
	& +\nabla_{\g}\sum_{t=1}^T \mathbb{E}_{q(c_{\star}^{(t)})} \left[\log H(\g)+ \langle \lambdab_{\ct}, \T(\g) \rangle \right]
	\label{eq:gradient}
\end{align}
where we have used $\g=\g(\xt)$ for brevity. The parameters are then updated as:
\begin{align}
	\g^{\text{new}} \leftarrow \g^{\text{old}} + \eta \nabla_{\g}\mathcal{\tilde{L}}(q(\mathbf{c}), \Theta) 
\end{align}
See Appendix \ref{apx:converg} for a discussion on the convergence of our algorithm.

The gradient term with respect to the determinant of the Jacobian $ \log|\J\g|$ deserves special attention. It is widely considered difficult to compute, and therefore, normalizing flows models are often used in literature in order to make the Jacobians more tractable. The problem with this approach is that, to our best knowledge, none of such flow models has universal \textit{function} approximation capabilities (despite some being universal \textit{distribution} approximators). This would restrict the possible set of nonlinear mixing functions that can be estimated, and is thus not practical for our purposes. Fortunately modern autograd packages such as JAX make it possible to calculate gradients of the log determinant Jacobian efficiently up to moderate dimensions (see Appendix \ref{apx:B}) -- this is the approach we take. Very recent, promising, alternative for computing the log-determinant is the relative gradient \citep{gresele_jacobian_2020} which could easily be implemented in our framework. Finally, notice that the second term \eqref{eq:gradient} is easy to evaluate since the expectation is just a discrete sum over the posteriors that we get from the E-step.

\subsection{COMMENT ON ESTIMATION FOR LONG TIME SEQUENCES} \label{sec:longtime}
The above estimation method may be impractical for very long time sequences since the forward-backward algorithm has computational complexity of $\mathcal{O}(TC^2)$. In such situations we can adapt the subchain sampling approach of \cite{foti_stochastic_2014}. This involves splitting up the full dataset into shorter time sequences and then forming minibatches over time. The resulting gradient updates would be biased and therefore a scaling term will be applied to them. The forward-backward algorithm applied to the subchains is also only approximate due to loss of information at the ends of the chains but the authors describe a technique to buffer the chains with extra observations to reduce this effect.

\subsection{COMMENT ON DIMENSION REDUCTION}
An important problem in applying our method on real data is dimension reduction. While in the theory above, we assumed that the number of independent components is equal to the number of observed variables, in many practical cases, we would like to have a smaller number of components than observed variables. We propose here two solutions for this problem.

The first solution, which is widely used in the linear ICA case, is to first reduce the dimension of the data by PCA, and then do ICA in that reduced space with the same dimensions of components and observed variables. In the nonlinear case, a number of nonlinear PCA methods, also called manifold learning methods, has been proposed and could be used for such a two-stage method. In particular, dimension reduction is achieved by even the very simplest autoencoders; recent work has developed the theory further in various directions \citep{maaten_visualizing_2008,vincent_stacked_2010}. This approach has the advantage of reducing the noise in the data, which is a well-known property of PCA, and allows us to separate the problem of dimension reduction from the problem of developing ICA algorithms. A possible drawback is that such dimension reduction may not be optimal from the viewpoint of estimating independent components.

The second solution is to build an explicit noise model into the nonlinear ICA model, following \citet{khemakhem_variational_2020}. Denote by $\mathbf{n}$ a random vector of Gaussian noise which is white both temporally and spatially and  of variance $\sigma^2$. Instead of the Eq.~(\ref{eq:mix}), we would define a mixing model as
\begin{align}
	\xt = \f(\st) + \mathbf{n}^{(t)}
	\label{eq:mixnoisy}
\end{align}
where the model of the components $\st$ is unchanged. We could then combine the variational estimation method presented by \citet{khemakhem_variational_2020} with the HMM inference procedure presented here. However, we leave the details for future work.

\section{IDENTIFIABILITY THEORY} \label{sec:theory}

In this section we provide identifiability theory for the model discussed in the previous section. As was discussed above, many deep latent variable models are non-identifiable. In other words, an estimation method such as the EM proposed above might not have a unique solution, or even a small number of solutions which are indistinguishable for any practical purposes.

Fortunately, we are able to combine previous nonlinear ICA theory with the identifiability of Hidden Markov Models to prove the identifiability of our combined model. Albeit our model being different from \citep{hyvarinen_nonlinear_2017}, \citep{hyvarinen_nonlinear_2019} and \citep{khemakhem_variational_2020}, the identifiability we reach is very similar. We also show that in the case of Gaussian independent components we can get exact identifiability up to linear transformation of the components. 

\subsection{DEFINITIONS}
In order to illustrate the relationship of our model's identifiability to earlier works in the area, we introduce the following definitions from \citet{khemakhem_variational_2020}
\begin{Definition}
	Let $\sim$ be the equivalence relation on $\Theta$. \eqref{eq:model_likeli} is said to be identifiable up to $\sim$ if
	\begin{align}
		p(\x^{(1)},\dots, \x^{(T)}; \Theta) = p(\x^{(1)},\dots, \x^{(T)}; \hat{\Theta}) \Rightarrow \Theta \sim \hat{\Theta}
	\end{align}
\end{Definition}

\begin{Definition} \label{def:sim}
	Let $\sim$ be the binary relation on $\Theta$ defined by:
	\begin{align}
		&(\f, \lambdab) \sim (\hat{\f}, \hat{\lambdab}) \leftrightarrow \nonumber \\
		& \exists \W, \mathbf{b} \mid \T(\g(\xt)) = \W \T(\hat{\g}(\xt)) + \mathbf{b} 
	\end{align}
	where $\W$ is an $NV \times NV$ matrix and $\mathbf{b}$ is an $NV\times1$ vector.
	
	If $\W$ is invertible, the above relation is denote by $\sim_W$, and if $\W$ is a block permutation matrix, it is denoted by $\sim_{\mathcal{P}}$. In block permutation, each block linearly transforms $\T_i(g_i(x_i))$ into $\T_j(\hat{g}_j(x_i)) \,\,$ with each $j$ corresponding to one, and only one, $i$.
\end{Definition}

\subsection{GENERAL RESULT}

Now we present our most general Theorem on identifiability. It will be followed by stronger results in the Gaussian case below.

\begin{Theorem} \label{T1}
Assume observed data is generated by a Hidden Markov Nonlinear ICA according to \eqref{eq:hmm_likeli} - \eqref{eq:model_likeli}. Also, assume:
	\begin{enumerate}[label=(\roman*)]
		\item The time-homogeneous transition matrix $\A$ has full rank and induces an irreducible \footnote{all states can be reached from every state} Markov chain with a unique stationary state distribution \label{t1:as1}
		\item The number of latent states, $C$, is known \label{t1:as2} and $C \geq NV + 1$
		\item There exists an $NV$ square matrix of the different states' parameters with respect to a pivot state 
			\begin{align}
				\widetilde{\Lambdab} = \begin{pmatrix} (\lambdab_{c=1}-\lambdab_{c=0})^T \\ \vdots \\ (\lambdab_{c=NV}-\lambdab_{c=0})^T \end{pmatrix} 
			\end{align}
			which is invertible. \label{t1:as4}
		\item The emission distributions for the different latent states $p(\xt|1; \thetab_1), \dots, p(\xt|C; \thetab_C)$ are linearly independent functions of $\xt$ \label{t1:as3}
		\item The non-linear mixing function $\f$ is bijective \label{t2:as5}
	\end{enumerate}
Then the model parameters $(\f, \lambdab)$ are $\sim_W$ identifiable.
\end{Theorem}

\begin{proof}
	Suppose we have
	\begin{align}
		p(\x^{(1)},\dots, \x^{(T)}; \Theta) = p(\x^{(1)},\dots, \x^{(T)}; \hat{\Theta})
	\end{align}

	Using assumptions \ref{t1:as1}-\ref{t1:as3}, we can invoke Theorem \ref{T:gassiat} and apply it to our model to get:
	\begin{align}
		&\hat{A}_{k,l} = A_{\sigma(k),\sigma(l)} \\
		&p(\x|k; \hat{\thetab}_k) = p(\x|\sigma (k); \thetab_{\sigma (k)}) \label{eq:marg_iden}
	\end{align}
	where superscript $t$ is dropped for convenience. For notational simplicity, and without loss of generality, we assume the components are ordered such that $k=\sigma(k)$. Substituting in $\eqref{eq:emission}$ we have:
	\begin{align}
		&|\J\hat{\g}(\x)| \frac{H(\hat{\g}(\x))}{Z(\hat{\lambdab}_{k})}\exp\{\langle \hat{\lambdab}_{k}, \T(\hat{\g}(\x)) \rangle\} \nonumber \\
		&=|\J\g(\x)| \frac{H(\g(\x))}{Z(\lambdab_{k})}\exp\{\langle \lambdab_{k}, \T(\g(\x)) \rangle \}
		\label{eq:equal_emissions}
	\end{align}
	for some latent state $k$. Recall from assumption \ref{t1:as4} that $C \geq NV+1$. We can thus take $C+1$ states and assign one of them, say $c=0$ as a pivot states. Taking logs of \eqref{eq:equal_emissions} for all the other states with respect to the pivot state gives $C$ equations of below form:
	\begin{align}
		&\langle (\lambdab_{k}- \lambdab_{1}), \T(g_i(\x)) \rangle + \log Z(\lambdab_{1})- \log Z(\lambdab_{k}) \nonumber \\
		&=\langle (\hat{\lambdab}_{k}- \hat{\lambdab}_{1}), \T(\hat{g}_i(\x)) \rangle + \log Z(\hat{\lambdab}_{1})- \log Z(\hat{\lambdab}_{c})
	\end{align}
	Collecting all the $C$ such equations, we can stack them into a linear system :
	\begin{align}
		\widetilde{\Lambdab} \T(\g(\x)) = \widehat{\widetilde{\Lambdab}}\T(\hat{\g}(\x)) + \boldsymbol{\beta}
	\end{align}
	where $\widetilde{\Lambdab}$ is the invertible square matrix defined in assumption \ref{t1:as4}, and the elements of $\widehat{\widetilde{\Lambdab}}$ are defined similarly, but no assumption about its invertibility is made. The constants that result from the sums of the log-normalizers are stacked to form $C\times 1$ vector $\boldsymbol{\beta}$. Multiplying both sides by $\widetilde{\Lambdab}^{-1}$ results in our desired form:
	\begin{align}
		&\T(\g(\x)) = \widetilde{\Lambdab}^{-1} \widehat{\widetilde{\Lambdab}} \T(\hat{\g}(\x)) + \widetilde{\Lambdab}^{-1} \boldsymbol{\beta} \nonumber \\
		&\T(\s) = \W \T(\hat{\g}(\x)) + \mathbf{b} \label{eq:identif1}
	\end{align}
	Recall that we defined the exponential families to be in minimal representation in Section \ref{sec:model_def}. It follows that we can find an arbitrary number of points such that the $V$ vectors formed by the sufficient statistic functions of each independent component $(T_{i,1}(s_i),\dots,T_{i,V}(s_i))$, are linearly independent. This can be done separately for each $s_i$. Additionally, as $s_i$ and $s_j$ can be changed independently, we can find for $i \neq j$ then $T_l(s_i)$ and $T_m(s_j)$ are linearly independent for all $l,m \in (1,\dots,V)$. Therefore, all elements of the vector $\T(\s)$ are linearly independent which implies that the square matrix $\W$ in \eqref{eq:identif1} is invertible.
\end{proof}

\subsubsection{Comments on the assumptions of Theorem 1}
The assumptions \ref{t1:as1}, \ref{t1:as2} are standard HMM assumptions. The assumption of a full rank transition matrix is non-standard but crucial here. Intuitively speaking, it allows the latent states to be distinguished from each other, while the irreducibility assumptions ensures that there is a single unique stationary state distribution.\footnote{technically one aperiodic state is also required. An aperiodic state is one which can be returned to after an irregular number of steps} Notice that these assumptions necessarily hold, for example, when the transition matrix is close to identity, as in a case where the states are strongly persistent.

The assumption that the real number of latent components is known, is valid in certain applications, and if not it could be estimated for instance be increasing the number of latent states between each estimation and then detecting the point at which increases in likelihood become marginal (the elbow method). Assumption \ref{t1:as4} is valid in practice as long as the parameters are generated randomly - in that case it almost surely holds as singular solutions will lie in a submanifold of lower dimension. The validity of assumption is less obvious \ref{t1:as3}, however, we will below prove that it holds, for instance, in the case of Gaussian independent components.

\subsection{IDENTIFIABILITY WITH GAUSSIAN INDEPENDENT COMPONENTS}
In this section, we first provide two lemmas which we use to prove the claim, already alluded to above, that assumption \ref{t1:as3} of Theorem \ref{T1} is satisfied for Gaussian components. Then, we prove that in this case a stronger form of identifiability can be reached as a special case of above results, namely that we get exact identification of components up to linear transformation. Together these results make a strong case for using Gaussian latent components in practical applications.  

We begin by stating two Lemmas (proofs in Appendix~\ref{apx:proofs}):

\begin{Lemma}
	Assumption \ref{t1:as3} of Theorem 1 requires the $C$ emission distributions defined by \eqref{eq:emission} to be linearly independent. A sufficient, and necessary, condition is that the $C$ conditional source distributions defined by \eqref{eq:sourc_dist} are linearly independent.
\label{lemma:lin_ind_transf}
\end{Lemma}

\begin{Lemma}
	Assume $K$ probability density functions of $N$ random variables $p_1(z_1, \dots, z_N), \dots, p_K(z_1, \dots, z_N)$, and that each factorizes across the variables: $p_k(z_1, \dots, z_N) =\prod_{i=1}^N p_k^{(i)}(z_i) \, \, \forall k \in \{1, \dots, K\}$. If the $K$ factorial density functions $p_1^{(i)}(z_i), \dots, p_K^{(i)}(z_i)$ are linearly independent for some $i \in \{1, \dots, N\}$, then the $K$ joint-density functions $p_1(z_1, \dots, z_N), \dots, p_K(z_1, \dots, z_N)$ are linearly independent.
	\label{lemma:lin_ind_factor}
\end{Lemma}

Based on these Lemmas, we can prove the following Theorem (proof in Appendix \ref{apx:proofs}):

\begin{Theorem} \label{the:lind}
	Assume that distributions of the independent components conditional on the latent state, as defined by \eqref{eq:sourc_dist}, are Gaussian parameterised by mean and variance.  Assume also that the means of the $C$ density functions are all different. Then the emission distributions, defined by \eqref{eq:emission}, are linearly independent, thus satisfying assumption \ref{t1:as3} in Theorem \ref{T1}.
\end{Theorem}

Finally, we have the following Theorem which proves a stronger form of identifiability---essentially recovering the components with minimum indeterminacy---of our Hidden Markov Nonlinear ICA model in the Gaussian case (proof in Appendix~\ref{apx:proofs})
\begin{Theorem} \label{the:exact} 
	Assume that the latent independent components have a conditionally Gaussian distributions, and assume hypotheses \ref{t1:as1}, \ref{t1:as2},\ref{t1:as4} and \ref{t2:as5} of Theorem~\ref{T1} hold, as well as the assumptions of Theorem \ref{the:lind}. Additionally assume that the mixing function $\f$ has all of it second-order cross derivatives, then the components in our Hidden Markov Nonlinear ICA model are exactly identified up to linear transformation.
\end{Theorem}
Notice that this proof and the identifiability result is similar to that in \citep{sorrenson_disentanglement_2020}, although our models are entirely different. These authors also prove a general version for other distributions with different sufficient statistics.

\section{EXPERIMENTS}
In this section we present results from our simulations on artificial non-stationary data. Code, written in JAX, is available at \url{github.com/HHalva/hmnlica}. 

\textbf{Dataset:}
We generated a synthetic dataset from the model defined in Section \ref{sec:model_def}. More specifically, the independent components are created from non-stationary Gaussian emission distributions of an HMM with $C$ discrete states -- the latent state determines the means and the variances of the independent components at each time point. The transition matrix was defined so that at each time-step there was a 99\% probability that the state didn't change and a 1\% probability the latent state switches to another state.\footnote{for context, the probability of staying in the same state for over 100 time steps with these numbers is around 37\%} If it switches to another state, it will always go to the next one `in line', where we define a circular ordering for the states. That is, we defined a circular repeating path for the latent state where transitions could only happen to two states such that the transition matrix is close to identity (Figure \ref{fig:sdata} illustrates this). These settings were chosen to ensure the HMM assumptions of Theorem \ref{T1} hold, as well as to reflect a situation where a relatively small number of states repeats over time with some interesting, non-random, temporal dynamics, including persistence to stay in the same state. The mean and variance parameters were chosen at random for each latent state before data generation so that assumption \ref{t1:as4} of Theorem \ref{T1} holds. Similarly to \cite{hyvarinen_nonlinear_2019}, the mixing function \eqref{eq:mix} was simulated with a randomly initialized, invertible \footnote{invertibility achieved by having all layers $N$ units wide and utilizing leaky ReLUs} multi-layer perceptron (MLP) -- this produced the observed data for our experiments. The sequences that were created are 100,000 time steps long. The number of latent states was set such that $C = 2N+1$, which ensures that the assumptions \ref{t1:as2} and \ref{t1:as4} were fulfilled. 

\begin{figure}[h]
	\centering
	\includegraphics[scale=0.8]{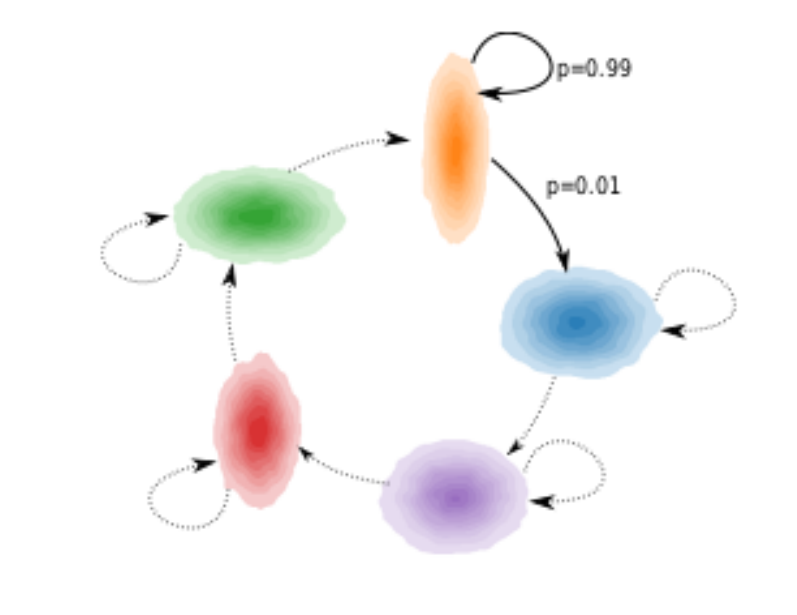}
\caption{An example of the independent components' distributions from a HMM where the number of components $N=2$ and the number of latent states $C=5$. The clusters are ordered to illustrated the dynamics of the hidden Markov model, in particular its circular property. The transition probabilities are the same throughout the data.}
\label{fig:sdata}
\end{figure}

\textbf{Model estimation:}
We estimate the (inverse) mixing function and distribution parameters of our Hidden Markov nonlinear-ICA model using the EM algorithm described in Section \ref{sec:estimation}. Mean and variance parameter estimates for the independent components are initialized randomly at the start of the EM algorithm. The inverse mixing function is parameterized with an MLP where the number hidden layers is set to match the number of data generating mixing layers. The gradient M-step are taken with the Adam optimizer \citep{kingma_adam_2017}. Random restarts were used to avoid inferior local maxima. Further, we found that a stochastic version of our algorithm (see Section \ref{sec:longtime}) converged faster -- thus, the experiments here have been run with 100 time-step long sub-sequences in minibatches of 64.

\textbf{Results -- independent component recovery:} After estimating the model parameters and independent components, a linear sum assignment problem is solved to optimally match each of the estimated components to one of the real ones. This is necessary as the ordering of the components is arbitrary. Mean absolute correlation coefficients over the resulting pairs of true and estimated components are then used to measure how well original components were recovered. This is the methodology taken in previous nonlinear ICA works \citep{hyvarinen_unsupervised_2016}.

Figure \ref{fig:results} shows the mean correlation between the estimated components for our model in comparison to TCL, which is the only other nonlinear ICA model for non-stationary data. For TCL, the data was split into 500 time-step long segments; 500 steps provided best performance relative to other computationally feasible options (100, 250, 750, 1000). We can see that our model outperforms TCL for all levels of nonlinearity. This validates our theoretical arguments that the TCL framework struggles with non-stationary data in which latent states (often a relatively small number) repeat over time since the segments it has access to don't correspond well with the true data generating process.

\textbf{Results -- temporal dynamics} Unlike previous models, Hidden Markov Nonlinear ICA is able to perform unsupervised clustering of latent states and to take into account the learned temporal dynamics in doing so. To estimate this ability, we ran the well-know Viterbi algorithm \citep{viterbi_error_1967} which finds the most likely path of latent states based on our estimated model. The results show that on average for $N=5$ and $C=11$ our model reaches near perfect classification accuracy in the linear ICA case, mean accuracy of around 80\% for $L=2$, and 68\% for $L=4$, thus clearly outperforming random chance level (figure in Appendix \ref{apx:pred}).   

\begin{figure}[h]
	\centering
	\includegraphics[scale=0.5]{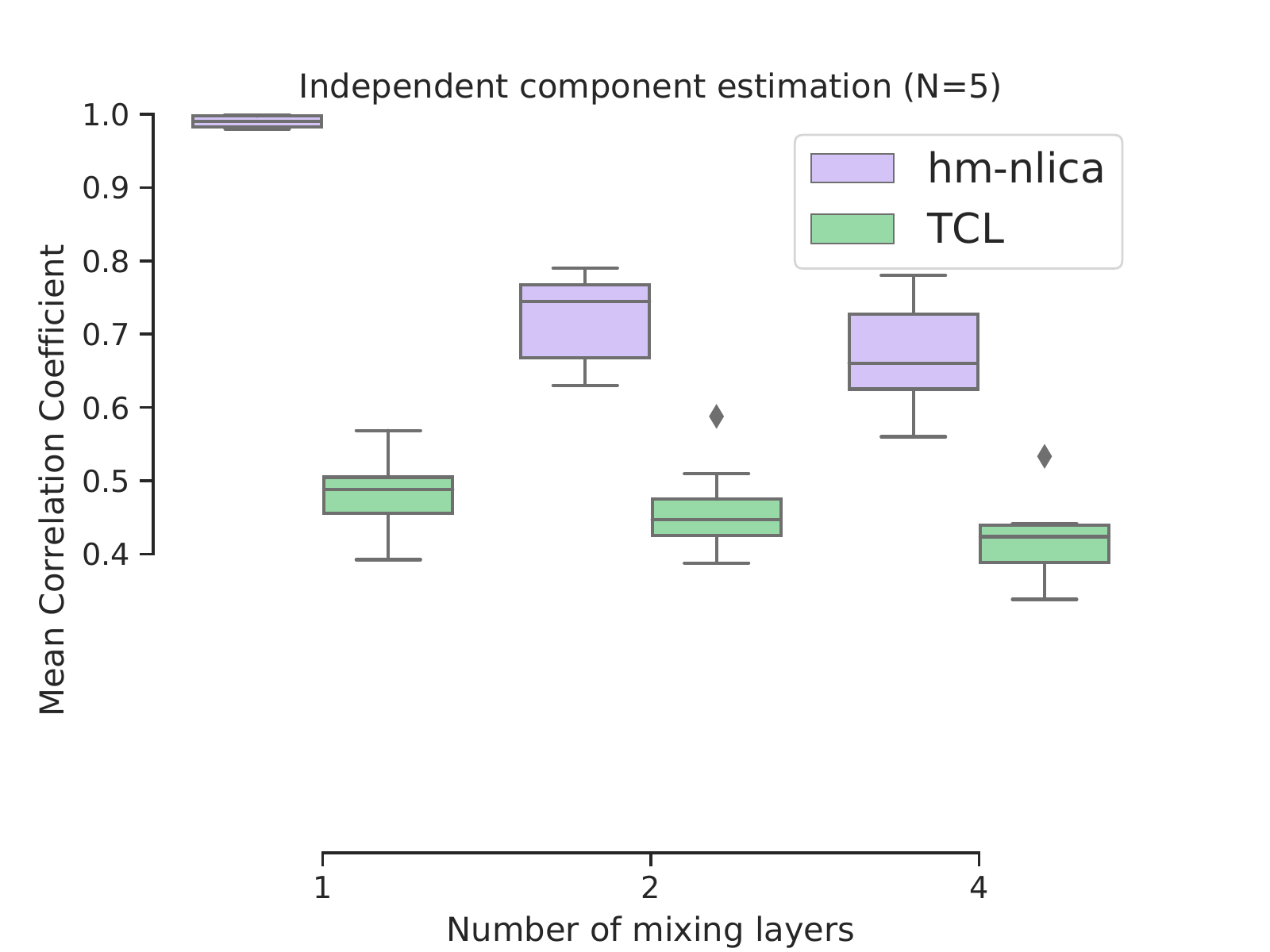}
\caption{Performance of our Hidden Markov nonlinear ICA vs. TCL in recovering the true sources for N=5 for our synthetic dataset. The amount of nonlinearity is controlled by number of hidden layers in the mixing MLP, so that $L \in [1, 2, 4]$.}
\label{fig:results}
\end{figure}

\section{CONCLUSION}
We proposed a framework nonlinear ICA based on a Hidden Markov Model of the temporal dynamics. This improves on existing nonlinear ICA methods in several ways. First, it removes the need for any arbitrary segmentation of the data as in TCL, which is likely to improve the estimation of the demixing function. Second, it leverages the fact that the nonstationary structure is often repeating with a limited number of hidden states, which not only reduces the computation by limiting the number classes, but again is likely to improve estimation of the demixing function. Third, our method estimates the underlying latent temporal dynamics, which are often interesting in their own right. We believe this in an important advance in order to apply nonlinear ICA methods on real data.

\subsubsection*{Acknowledgements}
The authors would like to thank Elisabeth Gassiat, Ilyes Khemakhem and Ricardo Pio Monti for helpful discussion. I.K.'s help with the experiments is also much appreciated. A.H. was supported by a Fellowship from CIFAR, and from the DATAIA convergence institute as part of the ``Programme d'Investissement d'Avenir" (ANR-17-CONV-0003) operated by Inria.

\newpage

\bibliography{nlica_hmm}
\bibliographystyle{apalike}

\clearpage
\onecolumn
\appendix


  \textit{\large Appendix for}\\ \ \\
  {\large \textbf{Hidden Markov Nonlinear ICA for Unsupervised Learning from Nonstationary Time Series}}\ (published at UAI 2020)\\
\vspace{.2cm}

\section{Note on the convergence of our estimation algorithm} \label{apx:converg}
Standard theory \citep{dempster_em_1977} shows that each EM iteration increases the likelihood, unless parameters are already at a zero-gradient point. Further, maxima of free-energy and likelihood coincide. This also holds under the gradient M-steps in our algorithm (with classical assumption of sufficiently small step size). Under suitable regularity conditions, theoretical limit of infinite data and universal approximation of the nonlinear transformation, combined with our identifiability proof, MLE guarantees convergence to correct parameters up to the equivalence class identified in our Theorem 2. In practice, however, these assumptions may not be satisfied --- for instance, parameters may approach a boundary point and likelihood tend to infinity. Random restarts and regularisation are common strategies to avoid these problems.

\section{Note on the compute time of the gradients of the logdet Jacobian} \label{apx:B}
We estimate the non-linear mixing function in our model using a multi-layer perceptron without any restrictions on it. As a consequence of the change of variable formula for probability densities, we have to calculate the gradient of the log-determinant of the Jacobian as part of our parameter updates. JAX, a new machine learning package that utilizes autograd, has the ability to calculate the Jacobian in just a single forward pass thus making the computations efficient for typical data dimensions. For our model, we can see that the compute time required for the log-determinant of the Jacobian starts to dominate as we approach 100 dimensions and above.
\begin{figure}[h]
	\centering
	\includegraphics[scale=0.5]{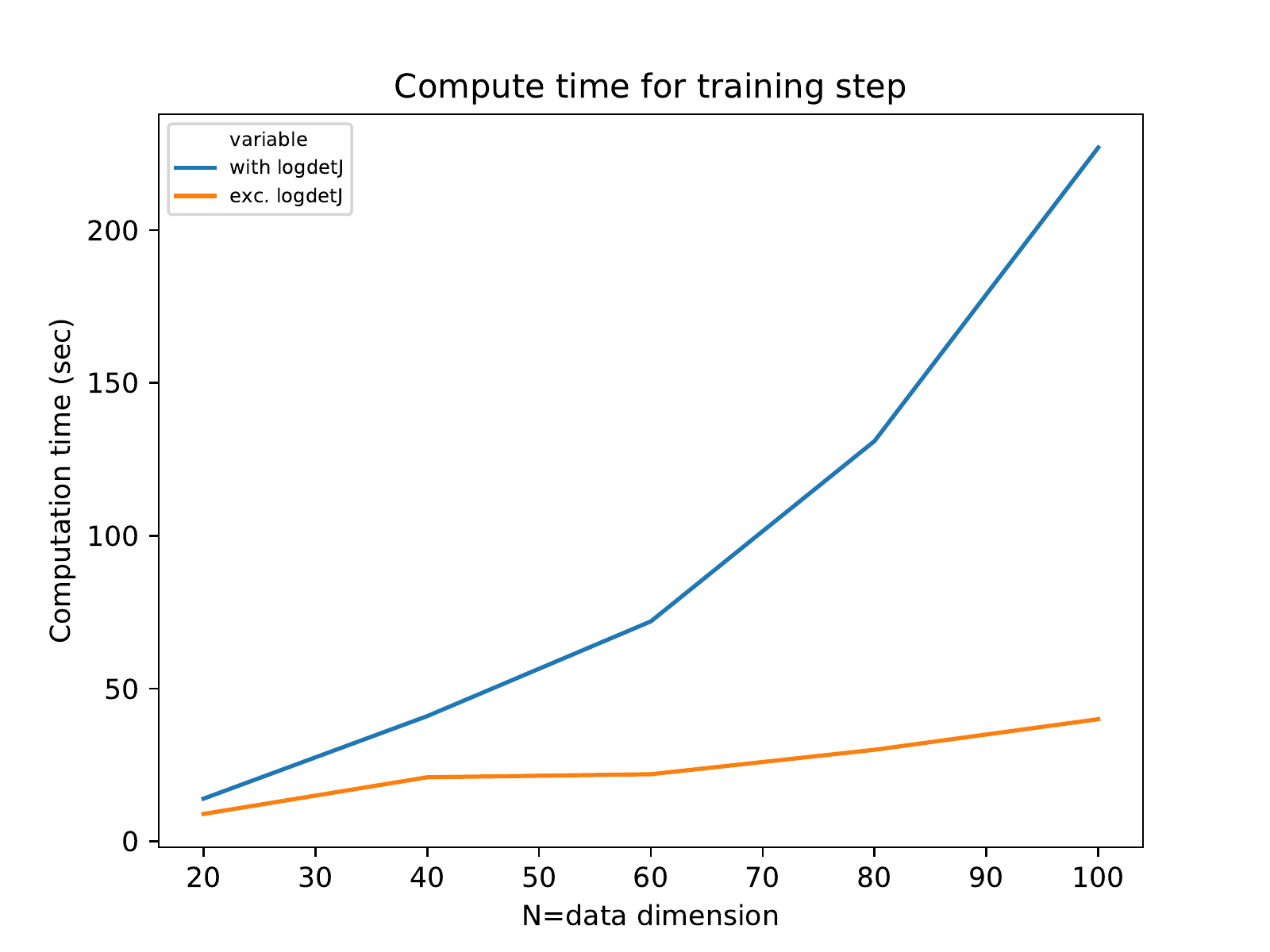}
	\caption{Average computation time over 100 epochs for computing the gradients of the function estimator in our model, including and excluding the log-determinant Jacobian term, in JAX. The function estimator is a four layer deep neural network where the width of the hidden units is always equal to $N$.}
\label{fig:compute_time}
\end{figure}

\newpage
\section{Proofs} \label{apx:proofs}

\textbf{Proof for Lemma \ref{lemma:lin_ind_transf}}
\begin{proof}
	Assume that we have linear independence of the $C$ conditional source distributions as defined by \eqref{eq:sourc_dist}. Then we have that:
	\begin{align}
		a_1 p_S(\s|1; \lambdab_{1}) + \dots + a_C p_S(\s|C; \lambdab_{C}) \equiv 0 \Rightarrow \mathbf{a}=\mathbf{0}
	\end{align}
	Above holds if we multiply it through by the Jacobian determinant of the mixing function as we have assumed bijectivity, that is:
	\begin{align}
		&a_1 |\J\g(\x)| p_S(\s|1; \lambdab_{1}) + \dots \nonumber \\
		&+ a_C |\J\g(\x)| p_S(\s|C; \lambdab_{C}) \equiv 0 \Rightarrow \mathbf{a}=\mathbf{0}
	\end{align}
	which is equivalent to:
	\begin{align}
		&a_1 |\J\g(\x)| p_S(\g(\x)|1; \lambdab_{1}) + \dots \nonumber \\
		&+ a_C |\J\g(\x)| p_S(\g(\x)|C; \lambdab_{C}) \equiv 0 \Rightarrow \mathbf{a}=\mathbf{0}
	\end{align}
	And therefore by \eqref{eq:emission} we have:
	\begin{align}
		&a_1 p_X(\x|1; \f, \lambdab_{1})+ \dots \nonumber \\
		&+ a_C p_X(\x|C; \f, \lambdab_{C}) \equiv 0 \Rightarrow \mathbf{a}=\mathbf{0}
	\end{align}
	So the emission distributions are linearly independent if the densities for the independence components are linearly independent across the $C$ different latent states. Necessity follows easily by the reverse of above argumentation. 
\end{proof}

\textbf{Proof for Lemma \ref{lemma:lin_ind_factor}}
\begin{proof}
	Assume linear independence of the $K$ joint-density functions for some subset of variables $z_i, \dots, z_{i+n}$, where $i \in \{1, \dots, N\}$ and $ 0 \leq n \leq N-i-1$, that is:
	\begin{align}
		&w_1 p_1(z_i, \dots, z_{i+n})+\dots+w_K p_K(z_i, \dots, z_{i+n}) \nonumber \\
		& = w_1 \prod_{j=i}^{i+n} p_1^{(j)}(z_j)+\dots+w_K \prod_{j=i}^{i+n} p_K^{(j)}(z_j) \equiv 0 \nonumber \\
		&\Rightarrow \w = \mathbf{0}
	\end{align}
	Now the linear independence for joint of $p_k(z_i, \dots, z_{i+n}, z_{i+n+1})$ requires:
	\begin{align}
		&w_1 p_1(z_i, \dots, z_{i+n}, z_{i+n+1})+ \nonumber \\
		&\dots+w_K p_K(z_i, \dots, z_{i+n}, z_{i+n+1}) \equiv 0 \nonumber \Rightarrow \w = \mathbf{0}
	\end{align}
	Using the factorial form of the joint, we can rewrite this as:
	\begin{align}
		&w_1 p_1^{(i+n+1)}(z_{i+n+1})p_1(z_i, \dots, z_{i+n}) + \dots \nonumber \\
		&+ w_K p_K^{(i+n+1)}(z_{i+n+1})p_K(z_i, \dots, z_{i+n}) \nonumber \\
		&\equiv 0 \Rightarrow \w = \mathbf{0}
	\end{align}
	If this didn't hold we could define $K$ constants $v_k:=w_k p_k^{(i+n+1)}(z_{i+n+1})$ such that:
	\begin{align}
		&v_1 p_1(z_i, \dots, z_{i+n}) + \dots + v_K p_K(z_i, \dots, z_{i+n}) \equiv 0 	
	\end{align}
	where the constants are not all zero which would contradict our original assumption. Thus it is sufficient to prove linear independence of $p_1^{(i)}(z_i), \dots, p_K^{(i)}(z_i)$, say for $i=1$, without loss of generality, and then apply the above induction step to guarantee linear independence of the $K$ joint-density functions $p_1(z_1, \dots, z_N), \dots, p_K(z_1, \dots, z_N)$ .
\end{proof}

\textbf{Proof for Theorem \ref{the:lind}}
\begin{proof}
	By Lemma \ref{lemma:lin_ind_transf} it is sufficient to prove the linear independence of the $C$ different conditional independent component density functions, rather than emission densities. And by Lemma \ref{lemma:lin_ind_factor} it suffices to prove this only for any one of the $N$ different independent components. In exponential family form, the density is written as (see Appendix \ref{apx:A}):
\begin{align}
	p(s_i|c)=\frac{1}{\sqrt{2\pi}}\frac{\exp\{\eta_{i, c, 1} s_i - \eta_{i, c, 2} s_i^2\}}{Z_{i, c}}
\end{align}
where $\eta_{i, c, 2} > 0 \, \forall c \in \{1, \dots, C\}$. We drop subscript $i$ for convenience. Consider:
\begin{align}
 &w_1\frac{1}{\sqrt{2\pi}}\frac{\exp\{\eta_{1, 1} s - \eta_{1, 2} s^2\}}{Z_{1}}+ \dots \nonumber \\
 &+w_C\frac{1}{\sqrt{2\pi}}\frac{\exp\{\eta_{C, 1} s - \eta_{C, 2} s^2\}}{Z_{C}} = 0
\end{align}

First assume, all the $\boldsymbol{\eta}_c$ are distinct. Also we can assume, without loss of generality, that the $C$ latent states are ordered such that $ \eta_{1, 2} < \eta_{2, 2} <\dots < \eta_{C, 2}$. We can divide all the terms with the first density to give:
\begin{align}
 &w_1+ w_2 \frac{Z_{1}}{Z_{2}}\exp\{(\eta_{1, 1}-\eta_{2, 1}) s + (\eta_{1, 2}-\eta_{2, 2}) s^2\}+\dots \nonumber \\
 &+ w_C \frac{Z_{1}}{Z_{C}}\exp\{(\eta_{1, 1}-\eta_{C, 1}) s + (\eta_{1, 2}-\eta_{C, 2}) s^2\}=0
\end{align}
taking $\lim_{s \rightarrow +\infty}$ of above gives $w_1=0$. Repeatedly performing this process for remaining terms eventually gives $\w=\mathbf{0}$. Consider now the opposite case in which all the $\boldsymbol{\eta}_c$ are equal. Then we have:
\begin{align}
 &w_1\frac{1}{\sqrt{2\pi}}\frac{\exp\{\eta_{1, 1}\}}{Z_{1}}+ \dots + w_C\frac{1}{\sqrt{2\pi}}\frac{\exp\{\eta_{C, 1}\}}{Z_{C}} = 0
\end{align}
If we re-order the terms such that $\eta_{1, 1}$ is the largest (recall we assumed that the means are different). We can again divide everything by this term, take $\lim_{s \rightarrow +\infty}$, and establish $w_1=0$ and repeat the process to get $\w=\mathbf{0}$. In the the final case where more than one component has the highest variance, but rest are unequal, (only the equality of the largest variances of the remaining terms matters), we can first perform the variance division followed by division by the largest mean, repeatedly until $\w = \mathbf{0}$. 
\end{proof}
\clearpage
\textbf{Proof for Theorem~\ref{the:exact}}
\begin{proof}
By Theorem~\ref{the:lind}, the above assumptions suffice for Theorem~\ref{T1} to hold. Next, note that the sufficient statistics of a Gaussian distribution are twice differentiable. This, combined with the assumption about the existence of $\f$'s cross-derivatives fulfils the conditions of Theorem 2 of \cite{khemakhem_variational_2020} and thus our model's parameters are $\sim_{\mathcal{P}}$ identifiable (as per Definition \ref{def:sim}). We therefore have:
\begin{align}
	\begin{pmatrix} s_i \\ s_i^2 \end{pmatrix} = \W_j \begin{pmatrix} g_j(\x) \\  g_j(\x)^2 \end{pmatrix} + \mathbf{b_i}
\end{align}
for some $i, j$. Hence, we have
\begin{align}
	&(w_{11}g_j(\x) + w_{12}g_j(\x)^2 + b_{11})^2 \nonumber \\
	&= w_{21}g_j(\x) + w_{22}g_j(\x)^2 + b_{21} \nonumber \\
	& w_{12}^2 z^4 + 2w_{11}w_{12}z^3 + (w_{11}^2 - w_{22}) z^2 - w_{21} z + b = 0
\end{align}
Above has to hold for all values of $z=g_j(\x)$. The trivial solution of all $w_{ij}=0$ is impossible as $\W$ would not be invertible. Therefore, it must be that $w_{12}=w_{21}=0$ and $w_{11}^2 = w_{22}$. Thus we have exact identification (up to linear transformation) $s_i = w_{ij} g_j(\x) + b_i$ for some constants $w_{ij}, b_i$.
\end{proof}

\section{Model with Gaussian independent components} \label{apx:A}
\begin{align}
	p(s_i|c)=&\frac{1}{\sqrt{2\pi\sigma_{i, c}^2}}\exp\{-\frac{1}{2\sigma_{i,c}^2}(s_i - \mu_{i,c})^2\}\\
	=&\frac{1}{\sqrt{2\pi\sigma_{i, c}^2}}\exp\{-\frac{1}{2\sigma_{i,c}^2}(s_i^2 - 2 s_i \mu_{i,c} + \mu_{i,c}^2)\}\\
	=&\frac{1}{\sqrt{2\pi\sigma_{i, c}^2}}\exp\{s_i \frac{\mu_{i,c}}{\sigma_{i,c}^2}  -s_i^2\frac{1}{2\sigma_{i,c}^2} - \frac{\mu_{i,c}^2}{2\sigma_{i,c}^2} \}\\
	=& Z_{i,c}^{-1}\exp\{s_i \frac{\mu_{i,c}}{\sigma_{i,c}^2}  -s_i^2\frac{1}{2\sigma_{i,c}^2}  \}
\end{align}
Therefore by independence of components:
\begin{align}
	p(\s|c)=& \exp\{\sum_{i=1}^N (s_i \frac{\mu_{i,c}}{\sigma_{i,c}^2}  -s_i^2\frac{1}{2\sigma_{i,c}^2})\}\prod_{i=1}^N  Z_{i,c}^{-1}\\
	       =& \exp\{\sum_{i=1}^N (s_i \frac{\mu_{i,c}}{\sigma_{i,c}^2}  -s_i^2\frac{1}{2\sigma_{i,c}^2})\} Z_{c}^{-1}
\end{align}
And change of variable gives:
\begin{align}
	p(\x|c)=& |\J\g(\x)| \exp\{\sum_{i=1}^N (g_i(\x) \frac{\mu_{i,c}}{\sigma_{i,c}^2}  -g_i(\x)^2\frac{1}{2\sigma_{i,c}^2})\} Z_{c}^{-1}\\
	=& |\J\g(\x)| \exp\{\langle\lambdab_c, \T(\g(x))\rangle \} Z_{c}^{-1}
\end{align}
where $\lambdab_c= \begin{bmatrix} \frac{\mu_{1,c}}{\sigma_{1,c}^2} \\ -\frac{1}{2\sigma_{1,c}^2} \\ \vdots \\ \frac{\mu_{N,c}}{\sigma_{N,c}^2} \\ -\frac{1}{2\sigma_{N,c}^2}\end{bmatrix}$ and $\T(\g(\x))= \begin{bmatrix}g_1(\x) \\ g_1^2(\x)\\ \vdots \\ g_N(\x) \\ g_N^2(\x)\end{bmatrix}$.

\section{Latent state prediction} \label{apx:pred}
\begin{figure}[h]
	\centering
	\includegraphics[scale=0.5]{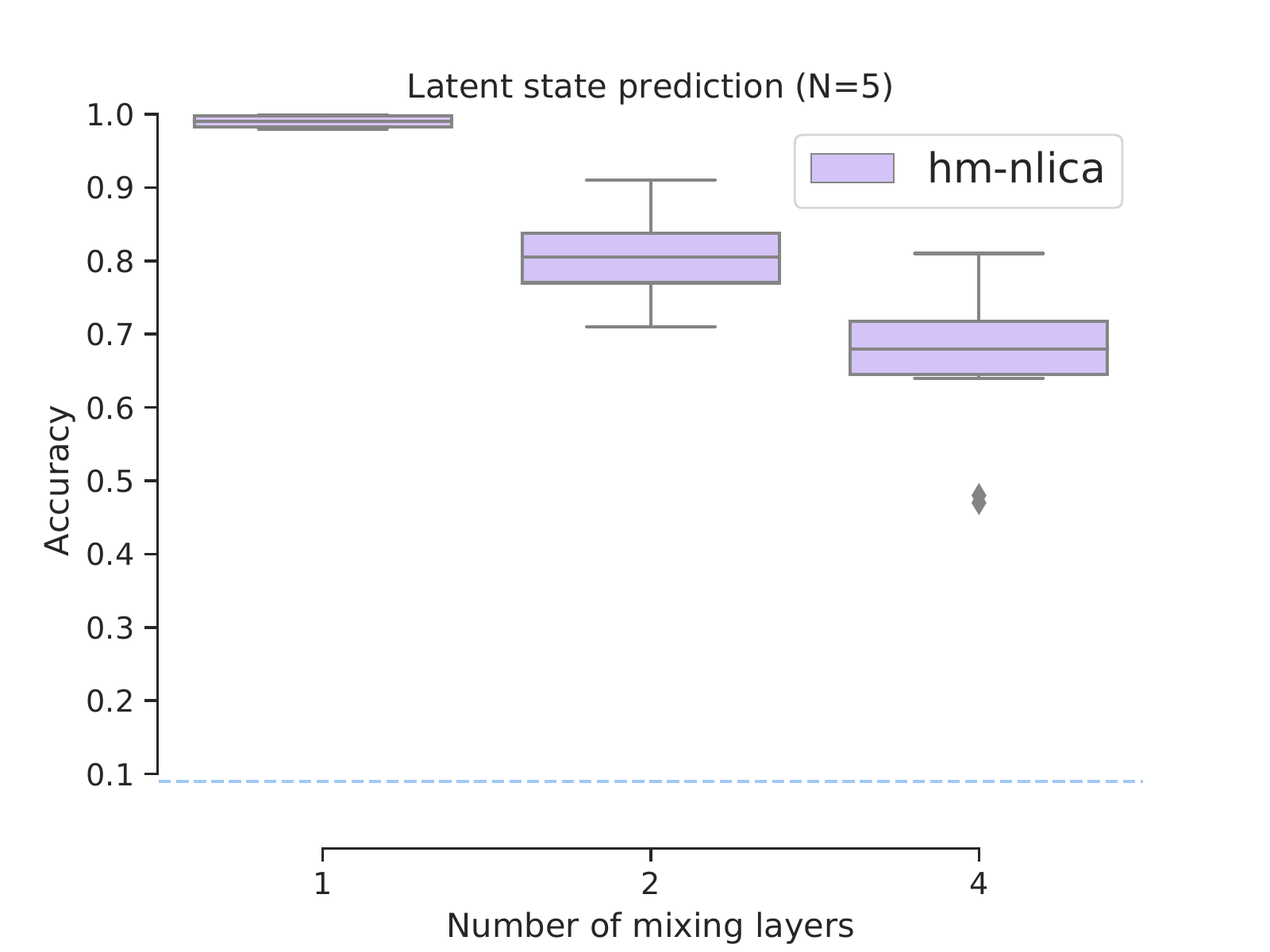}
	\caption{Performance of our Hidden Markov nonlinear ICA vs. chance level (dotted line = 0.09) for different levels of nonlinearity in latent state prediction. The number of latent states is $11 = 2N+1$}
	\label{fig:accresults}
\end{figure}

\end{document}